\documentclass[a4paper]{article}
\usepackage{amsmath,amssymb} 
\usepackage{tikz}
\usepackage{amsthm}
\usepackage{algorithm}
\usepackage{algorithmic}
\usepackage{authblk}

\newtheorem{mydef}{Definition}
\newtheorem{mylem}{Lemma}

\title{Implicit Modeling -- A Generalization of Discriminative and Generative Approaches}
\author{Dmitrij Schlesinger, Carsten Rother\footnote{This work was supported by German Federal Ministry of Education and Research (BMBF, 01IS14014A-D) and European Research Council (ERC) under the European Union's Horizon 2020 research and innovation programme (grant agreement No 647769). The computations were performed on an HPC Cluster at the Center for Information Services and High Performance Computing (ZIH) at TU Dresden.}}
\affil{Dresden University of Technology}
\date{}

\begin{document} 
\maketitle

\begin{abstract}
We propose a new modeling approach that is a generalization of generative and discriminative models. The core idea is to use an implicit parameterization of a joint probability distribution by specifying only the conditional distributions. The proposed scheme combines the advantages of both worlds -- it can use powerful complex discriminative models as its parts, having at the same time better generalization capabilities. We thoroughly evaluate the proposed method for a simple classification task with artificial data and illustrate its advantages for real-word scenarios on a semantic image segmentation problem.
\end{abstract} 

\section{Introduction}
We start with an overview that illustrates the proposed modeling approach in a broader context. In statistical modeling it is assumed that {\em there exists} a generating probability distribution that generates data and the goal is to search for a modeling probability distribution that fits the data adequately. Even if it is not necessary to model the generating probability distribution completely (e.g.~in case of discriminative models, see later) we prefer to start the considerations with a joint probability distribution $p(x,y)$, where $x$ is an observation and $y$ is an unknown associated label that we want to obtain, e.g.~the class label in a classification task.

In order to predict $y$ from $x$, a common approach is to first model the joint distribution $p(x,y)$, in particular by decomposing it as $p(x,y){=}p(x|y){\cdot}p(y)$. This is called generative modeling, where a prior over the labels $p(y)$, and a likelihood $p(x|y)$ are specified. Prediction is then carried out by employing $p(y|x){\propto}p(x|y){\cdot}p(y)$, obtained via Bayes' rule. Generative modeling often admits natural decompositions of the joint distribution: for example, in image segmentation $p(y)$ is a prior model over ``good'' segmentations, and $p(x|y)$ describes the image formation process. The learning of the unknown parameters is often carried out by maximizing the joint log-likelihood of the training data
\begin{equation}
\arg\max_{\theta_1,\theta_2} \sum_{t=1}^T\Bigl[\log p(y_t;\theta_1)+\log p(x_t|y_t;\theta_2)\Bigr] ,
\end{equation}
where $\bigl((x_t,y_t),t{=}1{\ldots}T\bigr)$ is the training sample and $\theta_1$ and $\theta_2$ are unknown parameters to be learned.

Alternatively, one can decompose $p(x,y){=}p(y|x){\cdot}p(x)$. While this kind of decomposition is often less intuitive, it has the property that $p(x)$ can be arbitrary and thus be ignored, since only $p(y|x)$ is required to predict $y$ from $x$. This is called discriminative modeling, where only the conditional $p(y|x)$ is specified. A prominent example of such a modeling approach are conditional random fields (CRF) \cite{Lafferty:2001:CRF:645530.655813}. The statistical learning is usually done by maximizing the conditional log-likelihood
\begin{equation}
\arg\max_{\theta} \sum_{t=1}^T \log p(y_t|x_t;\theta) .
\end{equation}
Discriminative models often exhibit excellent performance when training data is abundant. Besides they are less vulnerable to misspecification (see e.g.~\cite{mapmap}) -- i.e.~the situation that the ``true'' generating probability distribution is not contained in the modeled family. However, discriminative models  are prone to over-fitting when the parameters of the distribution $p(y|x)$ are learned from \emph{little data}.

\paragraph{Related works.} In order to alleviate the vulnerability of discriminative models to over-fitting, regularization in form of prior knowledge about the parameters is often used. In fact, the unknown parameters of the probability distribution are handled as additional random variables. Consequently, regularization to prevent over-fitting is often done in an ad-hoc manner by imposing simple mathematically convenient priors on model parameters in form of e.g.~sparsity constraints, quadratic penalties, etc. (see e.g.~\cite{Nowozin:2011:SLP:2185833.2185834} for an overview of related methods). Usually, the assumptions about the used parameter prior are far from being well founded (e.g.~Gaussian in the parameter space). Moreover, the design of such a prior is usually less intuitive as compared with the design of the original model.

Another line of work (e.g.~\cite{Bouchard:2004:TGD,Mccallum:2006:MCL}) proposes to combine generative and discriminative likelihood functions for learning. An example is \cite{Mccallum:2006:MCL}, where the so-called multi-conditional log-likelihood $\log\left(p(y|x;\theta)^\alpha p(x|y;\theta)^\beta\right)$ is maximized during training. Here two conditional models share parameters, and $\alpha,\beta$ are used to weight the contributions of the two conditionals. Note that the combination is done on the algorithmic level, rather than on the modeling one. Basically, the proposed multi-conditional log-likelihood does not correspond to any probability and hence is not statistically well-motivated.

In \cite{Bishop:2007:GDG} a learning objective is defined by using a pair of models each with its own set of parameters. Blending between generative and discriminative models is achieved by putting a prior on the two sets of parameters: fully independent parameters lead to a fully discriminative model, and constraining parameters to be identical recovers the fully generative model. Note that the method operates with {\em two different} models. Furthermore, the used probability distributions have to belong to the {\em same} family, since their parameters must be comparable.

\paragraph{Contribution.} We propose a modeling framework that has the following properties. First of all, it is a generalization of commonly used generative and discriminative modeling approaches. Blending between generative and discriminative extremes is achieved by appropriate design of the corresponding model parts. The framework is statistically well motivated -- the learning procedure searches for a joint probability distribution in a well (although not explicitly) defined family. We have no prior for the parameter (in contrast to e.g.~\cite{Nowozin:2011:SLP:2185833.2185834}). There is no model approximation as in \cite{Mccallum:2006:MCL}. In contrast to e.g.~\cite{Bishop:2007:GDG} we are not restricted in the choice of the conditional distributions. We can combine state-of-the-art discriminative models $p(y|x)$ with arbitrary generative components for $p(x|y)$, increasing thereby the generalization capabilities of the former.


\section{Implicit Models}

From now on we denote random variables as well as their ranges by capitalized symbols, i.e.~$X$ for observations and $Y$ for hidden variables. Particular values are denoted by $x\in X$ and $y\in Y$. Similarly, $\theta\in\Theta$ denotes a value of an unknown parameter, where $\Theta$ is the parameter space. Consequently, $p(X,Y;\theta)$ denotes a particular joint probability distribution, whereas $p(X,Y;\Theta)$ is a family of probability distributions.

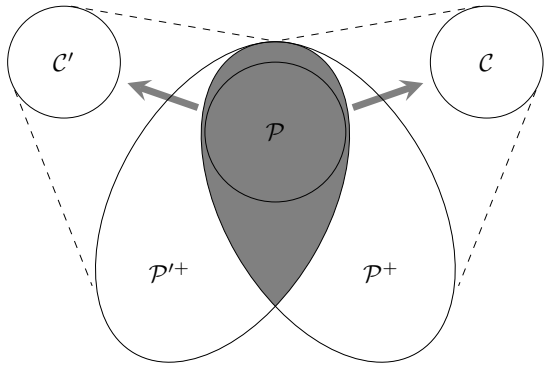
\begin{figure}
\centering
\resizebox{.6\linewidth}{!}{
\begin{tikzpicture}

\fill[gray] 
(0.75,-1) circle [x radius=1.5cm, y radius=2.5cm, rotate=30]
(-0.75,-1) circle [x radius=1.5cm, y radius=2.5cm, rotate=-30];
\fill[even odd rule][white]
(0.75,-1) circle [x radius=1.5cm, y radius=2.5cm, rotate=30]
(-0.75,-1) circle [x radius=1.5cm, y radius=2.5cm, rotate=-30];

\draw (0,0) node {${\mathcal P}$};
\draw (0,0) circle [x radius=1cm, y radius=1cm];

\draw (3,1) node {$\mathcal C$};
\draw [>=stealth,->,line width=1mm,color=gray] (1.1,0.35) -- (2.1,0.7);
\draw (3,1) circle [x radius=0.8cm, y radius=0.8cm];
\draw [dashed] (2.9,1.8) -- (0,1.3);
\draw [dashed] (3.7,0.6) -- (2.6,-2.2);
\draw (0.75,-1) circle [x radius=1.5cm, y radius=2.5cm, rotate=30];
\draw (1.5,-2) node {$\mathcal P^+$};

\draw (-3,1) node {$\mathcal C'$};
\draw [>=stealth,->,line width=1mm,color=gray] (-1.1,0.35) -- (-2.1,0.7);
\draw (-3,1) circle [x radius=0.8cm, y radius=0.8cm];
\draw [dashed] (-2.9,1.8) -- (0,1.3);
\draw [dashed] (-3.7,0.6) -- (-2.6,-2.2);
\draw (-0.75,-1) circle [x radius=1.5cm, y radius=2.5cm, rotate=-30];
\draw (-1.5,-2) node {$\mathcal P'^+$};
\end{tikzpicture}
}
\caption{\label{fig:families}The considered families of probability distributions (see text for explanations). The implicitly defined family of probability distributions of interest is shaded by gray color.}
\end{figure}

Let us recall the main advantages of generative and discriminative models, namely: (i) the generative models are more robust, i.e.~less vulnerable to overfitting, (ii) the discriminative models are more powerful and less vulnerable to misspecification for large data. The reason for such a behavior can be explained as follows. Consider a family of the joint probability distributions ${\mathcal P}{=}p(X,Y;\Theta)$ (see Fig.~\ref{fig:families}). Furthermore, consider the corresponding family of conditional (posterior) probability distributions ${\mathcal C}{=}p(Y|X;\Theta_1)$, i.e.~those which can be derived from members of $\mathcal P$. For example, let $\mathcal P$ be a family of Gaussian mixture models with two Gaussians having the same variance, i.e.~$p(x,y){=}p(y){\cdot}p(y|x)$, with the hidden variable $Y=\{1,2\}$ and $p(x|y;\mu,\sigma)={\mathcal N}(x;\mu_y,\sigma)$. Then $\mathcal C$ consists of all possible linear logistic regressions.

Note however, that at the same time there exists a corresponding family of prior distributions for observations $p(X;\Theta)$, i.e.~all those which can be obtained by marginalization of members of $\mathcal P$ over $Y$. The crucial assumption of the discriminative modeling is that the prior $p(X)$ can be arbitrary, which implicitly enlarges the family $\mathcal P$. Let us consider another family $\mathcal P^+$ of joint probability distributions, namely all those, whose conditional distributions are members of $\mathcal C$. For the above example with Gaussians, $\mathcal P^+$ will consists of all joint probability distributions that have a linear logistic regression as the posterior. Obviously, the family $\mathcal P^+$ is a superset of $\mathcal P$, i.e.~in $\mathcal P$ the family of priors $p(X;\Theta)$ is specified (i.e.~restricted), whereas in $\mathcal P^+$ the priors are arbitrary. Generative models work with $\mathcal P$ that is usually too restricted. It leads to worse recognition performance and vulnerability to model misspecification but better generalization capabilities. Discriminative models work in fact with $\mathcal P^+$ that is too general. Note that $\mathcal P^+$ can be understood as the cartesian product ${\mathcal C}{\times}p(X;\Theta)$, where $\Theta=\mathbb R^{|X|}$. Hence, $\mathcal P^+$ is extremely huge. If labeled training data is abundant, learning discriminative models is nevertheless able to ``recover'' the true generating probability distribution, which may be the member of $\mathcal P^+$ but not of $\mathcal P$. Hence, employing discriminative models alleviates the misspecification problem. However, it suffers from over-fitting for little data.

Our goal now is to develop an approach that allows to specify families of the joint probability distributions that are larger than $\mathcal P$ but more restricted than $\mathcal P^+$. The core idea is based on the following observation. In addition to the families $\mathcal P$, $\mathcal C$ and $\mathcal P^+$, consider the corresponding family of conditional probability distributions for observations given labels ${\mathcal C'}=p(X|Y;\Theta_2)$ -- i.e.~those which can be derived from members of $\mathcal P$. Furthermore, consider the family of joint probability distributions $\mathcal P'^+$ whose conditionals are members of $\mathcal C'$. Again, the family $\mathcal P'^+$ is a superset of $\mathcal P$. Hence, the intersection $\mathcal P^+\cap \mathcal P'^+$ is also a superset of $\mathcal P$. At the same time, this intersection is a subset of both $\mathcal P^+$ and $\mathcal P'^+$. Our goal is to learn joint models that are in this intersection. To summarize, we propose to \emph{implicitly} model the joint distribution $p(X,Y)$ by means of the two conditional distributions $p(Y|X)$ and $p(X|Y)$.

\begin{mydef}[Implicit Modeling]
Let two families $p(Y|X;\Theta_1)$ and $p(X|Y;\Theta_2)$ of conditional probability distributions be given. A joint probability distribution $p(X,Y)$ belongs to the family $p(X,Y;\Theta_1,\Theta_2)$ of interest if there exist $p(X)$, $p(Y)$, $\theta_1\in\Theta_1$ and $\theta_2\in\Theta_2$ satisfying
\begin{equation}\label{eq:impl1}
p(x)\cdot p(y|x;\theta_1)=p(y)\cdot p(x|y;\theta_2) \ \ \forall x\in X, y\in Y.
\end{equation}
Given that, the joint probability distribution is defined by $p(x,y;\theta_1,\theta_2)=p(x)\cdot p(y|x;\theta_1)$.
\end{mydef}
Obviously, the above conditions imply that the families $p(X)$ and $p(Y)$ of marginal probability distributions are not arbitrary but implicitly restricted by $\Theta_1$, $\Theta_2$ through \eqref{eq:impl1}. Hence, in further we will write $p(X;\Theta_1,\Theta_2)$ having in mind that $p(X)$ should satisfy \eqref{eq:impl1} (for $p(Y;\Theta_1,\Theta_2)$ analogously).

\subsection*{Weak implicit modeling}

In further we weaken conditions \eqref{eq:impl1}. The main motivation to do that is the following. In practice, we would not like ``to care'' about the modeling of the joint probability distribution at all. In other words, we would not like to {\em derive} the families of conditionals from a common family of joint probability distributions. We would like just to {\em design} two families of conditional probability distributions, hence implicitly defining the family of the corresponding joint ones. Moreover, we would like to be able to combine {\em nearly arbitrary} families of conditionals, that need not to have something in common from the modeling point of view. For example, we may want to use elaborated discriminative model (like e.g.~complex Conditional Random Fields or Convolutional Neural Networks) for $p(Y|X;\Theta_1)$, combining them with physically motivated forward models (e.g.~some rendering engines that produce images $x$ given a scene description $y$) for $p(X|Y;\Theta_2)$. However, given two particular conditionals $p(Y|X;\theta_1)$ and $p(X|Y;\theta_2)$, it may be the case that there is no priors $p(X)$ and $p(Y)$ satisfying \eqref{eq:impl1}. Moreover, it may be even the case that given two families $p(Y|X;\Theta_1)$ and $p(X|Y;\Theta_2)$, there is no pair of $\theta_1$ and $\theta_2$ for which there exist $p(X)$ and $p(Y)$ satisfying \eqref{eq:impl1}. Consider the following example. Let one of the families of conditionals, let say $p(X|Y;\Theta_2)$, is designed in such a way, that it leads to independent joint probability distribution, i.e.~$p(x|y;\theta_2)=p(x|y';\theta_2)$ for all $x$, all pairs $y$ and $y'$ and all $\theta_2$. Let the conditionals from the other family do not have this property. Obviously, according to the first family, the implicitly defined joint probability distribution should be independent, whereas according to the second one, it is not. Consequently, for such a pair of families there is no joint probability distribution satisfying \eqref{eq:impl1}.

Note that the conditions \eqref{eq:impl1} are posed for the marginal probabilities, but not for the joint ones. Hence, we weaken \eqref{eq:impl1} as follows:
{\sloppy
\begin{mydef}[Weak Implicit Modeling]
Let two families $p(Y|X;\Theta_1)$ and $p(X|Y;\Theta_2)$ of conditional probability distributions be given. A joint probability distribution $p(X,Y)$ belongs to the family $p(X,Y;\Theta_1,\Theta_2)$ of interest if there exist $p(X)$, $p(Y)$, $\theta_1\in\Theta_1$ and $\theta_2\in\Theta_2$ satisfying
\begin{eqnarray}\label{eq:impl2}
& & p(y)=\sum_{x\in X}p(x)\cdot p(y|x;\theta_1), \ \ \forall\ y\in Y , \nonumber \\
& & p(x)=\sum_{y\in Y}p(y)\cdot p(x|y;\theta_2), \ \ \forall\ x\in X .
\end{eqnarray}
Given that, the joint probability distribution is defined either by $p(x,y;\theta_1,\theta_2){=}p(x){\cdot} p(y|x;\theta_1)$ or by $p(x,y;\theta_1,\theta_2){=}p(y){\cdot} p(x|y;\theta_2)$ depending on the application.
\end{mydef}
}

Obviously, \eqref{eq:impl2} are necessary but not sufficient conditions for \eqref{eq:impl1}. The following lemma asserts the applicability of the weak implicit modeling for situations described above, namely if the considered families of conditional probability distributions are not derived from a common joint model, but designed in an application specific way independently from each other.

\begin{mylem}
Under some mild conditions there exists a pair of marginal probability distributions $p(X)$ and $p(Y)$ satisfying \eqref{eq:impl2} for any pair of conditional probability distributions $p(Y|X;\theta_1)$ and $p(X|Y;\theta_2)$ (see proof in Appendix \ref{appendix1}).
\end{mylem}

It is also easy to see that conditions \eqref{eq:impl2} implicitly define {\em two} (possibly) different joint probability distributions: one being defined by $p(x,y;\theta_1,\theta_2)=p(x;\theta_1,\theta_2)\cdot p(y|x;\theta_1)$ and the other defined by $p(x,y;\theta_1,\theta_2)=p(y;\theta_1,\theta_2)\cdot p(x|y;\theta_2)$. In practice, one should decide which of these two is of interest. In most cases we are interested in recognition at the end, for which only the posterior $p(Y|X)$ is necessary for inference. Therefore it seems reasonable to define the family of joint probability distributions so that its induced posterior coincides with the conditional $p(Y|X;\theta_1)$ that we model. Hence, in further we define the joint by $p(x,y;\theta_1,\theta_2)=p(x;\theta_1,\theta_2)\cdot p(y|x;\theta_1)$. Of course, in other application scenarios it might be the other case.

We would like to point in that although we weakened the original conditions \eqref{eq:impl1}, the family of joint probability distributions $p(x,y;\theta_1,\theta_2)=p(x;\theta_1,\theta_2)\cdot p(y|x;\theta_1)$ satisfying \eqref{eq:impl2} is still considerably restricted as compared to the standard discriminative approach. Remember that in the latter the marginal $p(X)$ is completely free, whereas in the former the set of $p(X)$ is restricted by \eqref{eq:impl2}.

The proposed modeling approach is neither generative nor discriminative (in the meaning described at the very beginning). It is closer to a generative one, in spirit, since we are working with joint probability distributions (although not explicitly defined). At the same time it is easy to see that both fully generative and discriminative models are special cases of the weak implicit modeling, by choosing the families of the conditional distributions as follows. When $p(X|Y;\Theta_2)$ does not depend on $Y$, i.e.~$p(x|y){=}p(x|y')$ for any $x$, $y$, $y'$ and $\theta_2$, we obtain a fully discriminative model. Similarly, we are in the generative extreme if $p(Y|X;\Theta_1)$ does not depend on $X$. On the other hand, if $p(Y|X;\Theta_1)$ is deterministic, i.e.~it is a mapping $X{\rightarrow}Y$, we deal with discriminant functions. Hence, the whole spectrum of modeling approaches fits into a common framework illustrated in Fig.~\ref{fig:frame}. The axes correspond to the ``weakness'' of the constituents. Commonly used fully discriminative and fully generative approaches occupy relatively small areas along the axes. We attempt to explore the area in the middle, among other things, to push the discriminative approaches towards the direction depicted by the blue arrow in the figure, increasing their robustness to overfitting.

\begin{figure}
\centering
\resizebox{.8\linewidth}{!}{
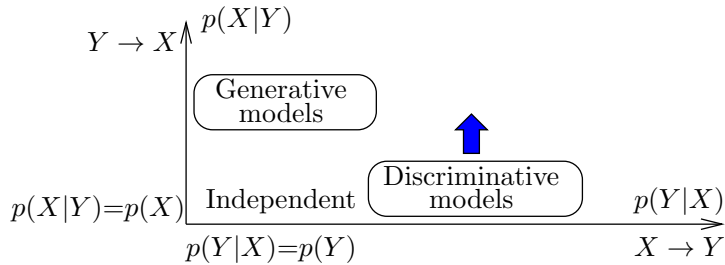
}
\caption{\label{fig:frame}A common framework of generative, discriminative and weak implicit modeling approaches. Axes correspond to the weakness of the corresponding parts. ``$p(Y|X)=p(Y)$'' means that $p(Y|X)$ does not depend on $X$, ``$X\rightarrow Y$'' means a deterministic mapping. For $p(X|Y)$ analogously.}
\end{figure}

The ``weakness'' of the used conditional probability distributions (and hence the position of an approach in the proposed framework in Fig.~\ref{fig:frame}) can be controlled by choosing how strongly for example $p(X|Y)$ depends on $Y$. In practice it can be done e.g.~in the following manner. Consider a distribution defined as $p(x|y)\propto\exp(\alpha E(x,y)+E(x))$ with some energy functions $E(\cdot)$. Setting $\alpha$ to zero leads to independence, a strong $\alpha$ corresponds to a deterministic mapping, i.e.~there is a (unique) value of $y$ for each $x$. If necessary, such a controlling factor can be set manually, weighting the ``weakness/importance'' of the corresponding family. On the other hand, it can also be considered as a usual unknown parameter and learned from data.


\section{Learning Implicit Models}
\label{sec:learning}

Given a training sample of observation-label pairs $\bigl((x_t,y_t),t{=}1{\ldots}T\bigr)$, we aim to learn our implicit model by maximizing the joint log-likelihood
\begin{eqnarray}
\lefteqn{\arg\max_{\theta_1,\theta_2} \sum_{t=1}^T \log p(x_t,y_t;\theta_1,\theta_2) =} \nonumber \\
& & =\arg\max_{\theta_1,\theta_2}\sum_{t=1}^T \Bigl[\log p(x_t;\theta_1,\theta_2)+ \log p(y_t|x_t;\theta_1)\Bigr] .
\label{eq:task2}
\end{eqnarray}
For gradient-based optimization, we need to differentiate \eqref{eq:task2} wrt.~$\theta_1,\theta_2$. It is straightforward to do that for the conditional $p(y|x;\theta_1)$. In contrast, the differentiation of the first addend in \eqref{eq:task2} is less trivial because there is no closed-form expression for it, i.e.~the marginal distribution $p(X;\theta_1,\theta_2)$ depends {\em implicitly} on both $\theta_1$ and $\theta_2$.

We model $p(X;\theta_1,\theta_2)$ by means of a generation process that samples observations according to the desired probability distribution. Let us assume an (infinite) Markov chain that generates sequences $(x^0,y^0,x^1,y^1,\ldots, x^{n-1},y^{n-1},x^n)$, $n{\rightarrow}\infty$ by sampling from the corresponding conditional distributions $p(Y|X;\theta_1)$ and $p(Y|X;\theta_2)$ respectively, i.e.~$y^i$ is drawn from $p(Y|X{=}x^i;\theta_1)$ and $x^i$ is drawn from $p(X|Y{=}y^{i-1};\theta_2)$. Obviously, stationary distribution of such a Markov chain satisfies \eqref{eq:impl2} (it follows directly from Lemma 1, see Corollary in Appendix \ref{appendix1}), i.e.~our $p(X;\theta_1,\theta_2)$ is the stationary distribution of this Markov chain. Hence, we model $p(x_t;\theta_1,\theta_2)$ with the probability that a Markov chain as described above generates the training example $x_t$.

For better readability we omit here the detailed derivation of the algorithm for optimizing \eqref{eq:task2}. More detailed explanations can be found in appendix \ref{appendix2}. In short: the probability $p(x_t;\theta_1,\theta_2)$ is obtained by marginalization over all sequences $(x^0,y^0,x^1,y^1,\ldots, x^{n-1},y^{n-1},x^n{=}x_t)$ generated by the Markov chain considered above. This summation is infeasible especially if $x$ and $y$ are complex by themselves, e.g.~images and labelings (the state space of the Markov chain is of huge dimensionality). Hence, we use stochastic gradient ascent for approximation -- instead to marginalize over all sequences, we draw just one from the Markov chain. To summarize, we need to carry out two steps to perform stochastic gradient ascent for optimizing \eqref{eq:task2} (on the example of just one training example $(x^\ast,y^\ast)$): (i) Generate a sequence of $x$-s and $y$-s starting from the training example $x^\ast$ according to the probability distributions $p(Y|X;\theta_1)$ and $p(X|Y;\theta_2)$ given the current model parameters, and (ii) consider certain generated pairs as the ``additional labeled training examples'' for computing gradients wrt.~$\theta_1$ and $\theta_2$. One stochastic gradient ascent step for one training example $(x^\ast,y^\ast)$ is summarized in Algorithm \ref{alg:alg1} and illustrated in Fig.~\ref{alg:alg} (for many examples their gradients should be averaged). Thereby, we assume that the conditional probability distributions of interest are both members of the exponential family, i.e.~they can be written in form e.g.~$p(y|x;\theta_1)\propto \exp\bigl(\langle\eta_1(x,y),\theta_1\rangle\bigr)$, where $\eta_1(x,y)$ are sufficient statistics. Hence, the necessary gradients are just differences of the corresponding sufficient statistics.

Interestingly, the algorithm is somewhat similar to the standard conditional likelihood learning. More precisely, if we omit the sampled $\tilde x$ and $\hat y$ and use only the pairs $(x^\ast,y^\ast)$ and $(x^\ast,\tilde y)$ for updating the gradient, it would be exactly the conditional likelihood. The crucial difference here is that not only pairs from the training sample are used for learning $p(Y|X;\theta_1)$ but also pairs that are generated by the current conditional probability distributions $p(Y|X;\theta_1)$ and $p(X|Y;\theta_2)$, which implicitly define the current joint probability distribution $p(X,Y,\theta_1,\theta_2)$.

\newcommand{\imrow}[1]{
\noindent
\includegraphics[width=0.19\textwidth]{./pics/#1_6_y2_c.png}\hfill
\includegraphics[width=0.19\textwidth]{./pics/#1_4_x1.png}\hfill
\includegraphics[width=0.19\textwidth]{./pics/#1_3_y1_c.png}\hfill
\includegraphics[width=0.19\textwidth]{./pics/#1_1_x0.png}\hfill
\includegraphics[width=0.19\textwidth]{./pics/#1_0_y0_c.png}\hfill
}

\begin{algorithm}[t]
\caption{\label{alg:alg1} Stochastic gradient ascent step for \eqref{eq:task2} (one training example) }
\begin{algorithmic}
\STATE {\bf Input:} Training example $(x^\ast,y^\ast) \equiv (x_t,y_t)$,\\
current parameters $\theta_1,\theta_2$
\STATE {\bf Output:} Updated model parameters $\theta_1',\theta_2'$
\STATE Sample $\tilde y$ according to $p(Y|X{=}x^\ast;\theta_1)$
\STATE Sample $\tilde x$ according to $p(X|Y{=}\tilde y;\theta_2)$
\STATE Sample $\hat y$ according to $p(Y|X{=}\tilde x;\theta_1)$
\STATE Compute sufficient statistics $\eta$ for the generated pairs
\STATE Compute gradients and do steepest ascent with a small step size $\lambda$:

\vskip -4ex
\begin{eqnarray*}
& & \theta_1'=\theta_1+\lambda\cdot\bigl(\eta_1(\tilde x,\tilde y)-\eta_1(\tilde x,\hat y)+ \eta_1(x^\ast,y^\ast)-\eta_1(x^\ast,\tilde y)\bigr) \\
& & \theta_2'=\theta_2+\lambda\cdot\bigl(\eta_2(x^\ast,\tilde y)-\eta_2(\tilde x,\tilde y)\bigr)
\end{eqnarray*}
\end{algorithmic}
\end{algorithm}

\begin{figure*}
\begin{center}
\resizebox{.8\linewidth}{!}{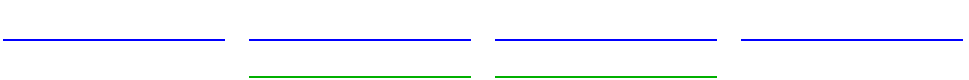}

\vspace{1ex}
\imrow{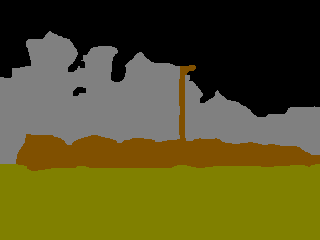}

\vspace{1ex}
\resizebox{.8\linewidth}{!}{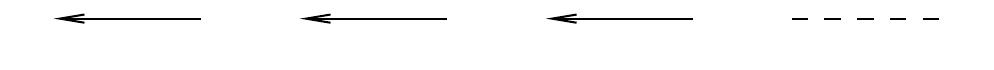}
\end{center}
\caption{\label{alg:alg}
Illustration of a single step of stochastic gradient ascent. The examples in the chain are generated in the  direction shown by the black arrows using the current conditional probability distributions. The pairs marked blue are used to learn $p(Y|X;\theta_1)$. The gradient is thereby obtained as the difference of sufficient statistics for the corresponding pairs. The pairs marked green are used to learn $p(X|Y;\theta_2)$. Images are taken from our experiment with semantic segmentation (see the next section). More qualitative results can be found in Appendix \ref{appendix3}.}
\end{figure*}

At this point we would like to note that our algorithm reminds to some extent on the Contrastive Divergence \cite{Hinton2002}, since we also generate chains of a finite length starting from training examples in order to get examples for gradient calculation. We would like to point in that our algorithm is by no means a variant of Contrastive Divergence, because the latter just performs another task. It is designed to draw examples from a {\em joint} probability distribution given {\em explicitly}. It is done by alternate sampling from the conditional probability distributions that are {\em derived} from the target joint one (i.e.~Gibbs Sampling). Usually, both conditional probability distributions are simply given by means of an energy function that is {\em the same} for both conditionals. Given a generated chain, the {\em last} example is assumed to be drawn according to the needed joint probability distribution and is used for gradient calculation. In contrast, our algorithm does not need single examples drawn from the joint probability distribution, but the whole chains that are generated according to the certain rules (see Appendix \ref{appendix2} for details). Thereby {\em all} generated examples are used for gradient calculation. Moreover, we do not specify the joint probability distribution at all. Only conditional probability distributions are given, that are not derived from a common joint one but just belong to pre-defined families, that might have nothing in common (no similar structure, no common parameters, etc.).


\section{Experiments}

\paragraph{Classification example with artificial data.} We start with an artificial example to illustrate our modeling approach. First, we define a generating joint probability distribution, from which training and test samples are drawn. Thus we define it in a generative manner, i.e.~$p(x,y)=p(y)\cdot p(x|y)$. The variable $Y$ is discrete and can take three values $y\in\{1,2,3\}$ (we will call them ``classes''), the prior probability distribution of classes is uniform. The observation $X$ is a real number, the conditional probability distribution $p(X|Y)$ consists of Gaussians (one per class) of the same variance but with different mean values. To summarize, the generating model is
\begin{equation}
p(x,y)=\frac{1}{3}\cdot\frac{1}{\sqrt{2\pi}\sigma}\exp\left[-\frac{(x-\mu_y)^2}{2\sigma^2}\right] .
\end{equation}
In particular, we use the variance $\sigma=1$ and mean values $\mu_y=-1,0,1$ for three classes in this experiment.

The learned implicit model consists of two parts. The first one (we refer it as discriminative part) is the conditional probability distribution $p(Y|X)$ of classes given an observation, for which we choose quadratic logistic regression $p(y|x){\propto}\exp[a_y{\cdot}x^2{+}b_y{\cdot} x{+}c_y]$, $y{\in}\{1,2,3\}$. We intentionally use a model that may over-fit (linear logistic regression would be sufficient for the true generating model as described above). The second part (called the generative one) is the conditional probability distribution for observations given classes, for which we use one Gaussian per class. To be consistent with the above notations we use the exponential family like parameterization $p(x|y){\propto}\exp[d_y{\cdot} x^2{+}e_y{\cdot} x]$. Note that this model may over-fit as well, as we allow the Gaussians to have different $d_y$ that correspond to different variances.

We perform several experiments with an increasing size of training samples to analyze the generalization capabilities of the models and learning approaches that we consider. Each experiment consists of the following: first, a training sample of a particular size is drawn from the generating probability distribution. After that learning is performed, where we compare the following schemes:
\begin{enumerate}
\item {\bf Standard conditional likelihood}, which we consider as the baseline. In addition, we perform learning with a quadratic regularizer for model parameters -- using ``weak'' (a small weighting constant) and ``strong'' regularization.
\item {\bf Learning of the implicit model} as described in the previous section (we refer it as ``implicit learning'').
\end{enumerate}
For the learned model we define the classifier to be the Maximum A-posteriori decision and compute two measures:
\begin{enumerate}
\item {\bf Test error rate.} The (average) error rate on a very large test sample, i.e.~the Bayesian risk for the learned classifier.
\item {\bf Risk difference.} Additionally, we compute the absolute value of the difference between the error rate on the training sample (the empirical risk) and the error rate on the test sample (the Bayesian risk, when the test sample is large enough). We would like to emphasize that this measure is crucial to adequately assess the generalization capabilities of the different models and learning strategies, because it represents a kind of ``guarantees'' for predicting unseen data. Smaller risk difference means thereby better generalization capability.
\end{enumerate}
\begin{figure}
\includegraphics[width=0.49\linewidth]{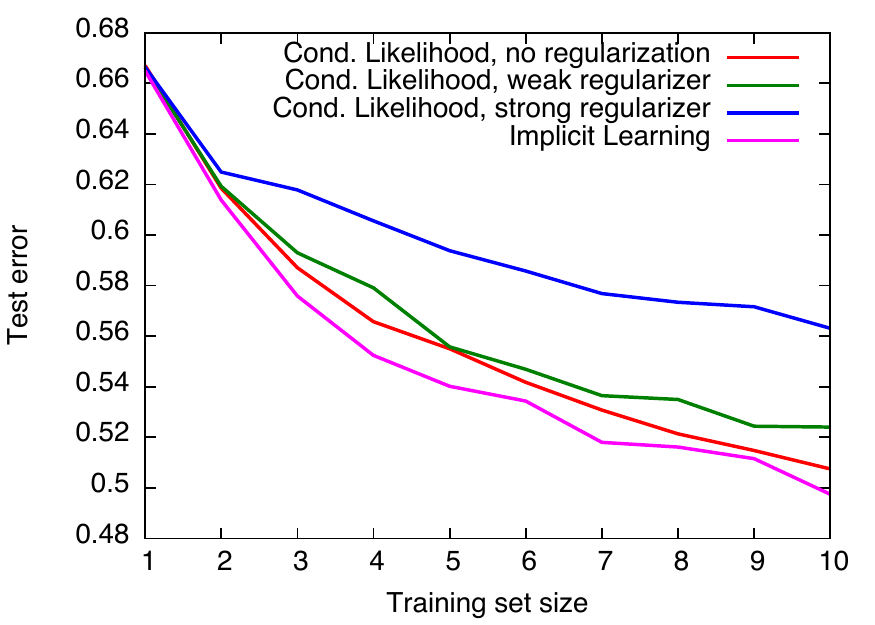}\hfill
\includegraphics[width=0.49\linewidth]{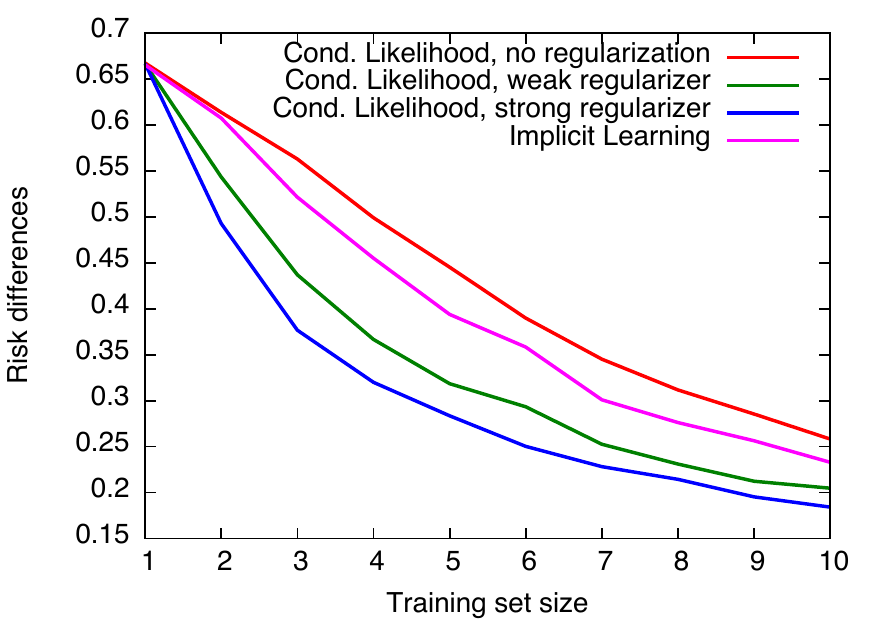}

\includegraphics[width=0.49\linewidth]{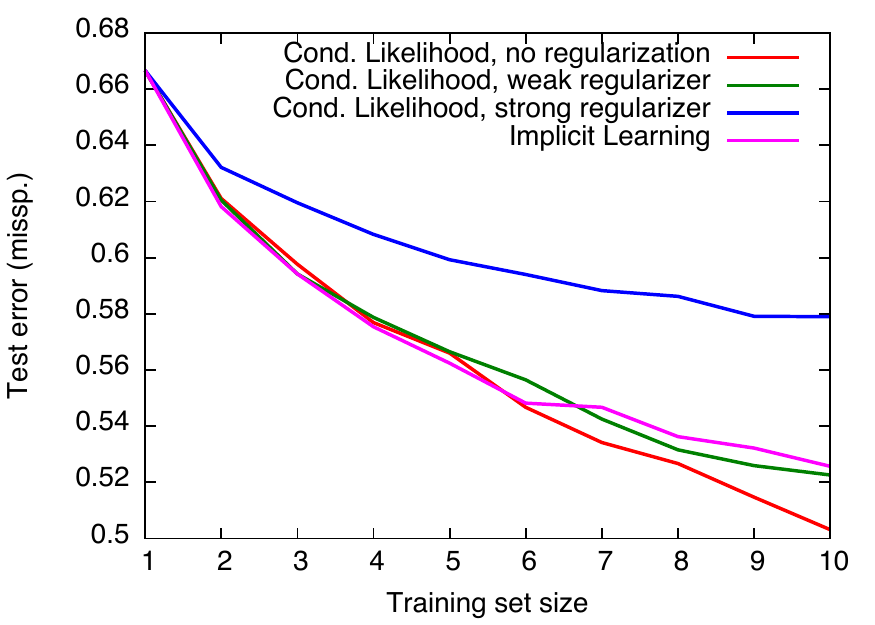}\hfill
\includegraphics[width=0.49\linewidth]{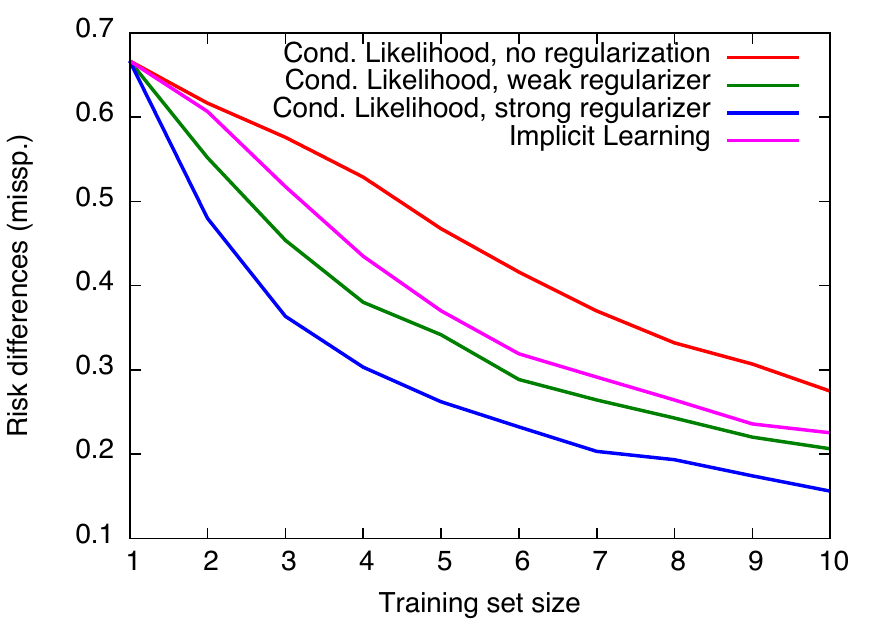}
\caption{\label{fig:art} Results for experiments with artificial data. Top row -- experiment without misspecification, bottom row -- with misspecification. Left: test error rates, right: risk differences.}
\end{figure}
For each particular size of the training set we repeat the experiment several times and report average measurements. Fig.~\ref{fig:art} (top row) shows the dependencies of the considered measures on the size of the training set. First of all, one can clearly see that the implicit learning consistently outperforms the standard conditional likelihood approach (without regularization) with respect to both measures. With respect to the recognition rate the implicit learning is also consistently better as compared to the regularized likelihood (for both weak and strong regularization). The most robust learning is the conditional likelihood with strong regularization. It is easy to explain by the fact that both the empirical and the Bayesian risks in this situation are worse as compared to other methods (see the blue line in Fig.~\ref{fig:art} (top, left)). Hence, the usual trade-off ``performance vs.~stability'' is clearly seen. 

The focus in the previous experiment was mainly to analyze the behavior of the various schemes with respect to over-fitting and generalization. There is also another (to some extent ``opposite'') problem that is often crucial in many real applications, namely the problem of misspecification. The situation is that the true (generating) probability distribution is not contained in the family of modeling probability distributions. It is well known that discriminative modeling performs better for larger training samples (in comparison to generative modeling), because the set of modeling probability distributions is larger. The aim of the following experiment is to study the influence of the introduced generative part on the behavior of the learning with respect to misspecification for smaller training samples. To address this problem we simulate misspecification in the following way. First, we restrict the generative part $p(X|Y)$ of our implicit model to be Gaussians with the {\em same} variance. At the same time, in the generating model we use Gaussians of {\em different} variances. Note that the posterior $p(y|x){\propto}\exp[a_y{\cdot} x^2{+}b_y{\cdot} x{+}c_y]$ alone is still well-specified. The generative part $p(X|Y)$ is not indeed. Consequently, the family of joint probability distributions that are modeled implicitly by the pair $p(X|Y)$ and $p(Y|X)$ is misspecified as well. In other words, the usage of such a misspecified $p(X|Y)$ considerably restricts the family of modeling joint probability distributions. We perform the same kind of experiments as done previously, the results are shown in Fig.~\ref{fig:art} (bottom row). As expected, the implicit learning performs slightly worse with respect to the recognition rate as compared to the case without misspecification. At the same time, the discriminative learning is better for larger training samples. Nevertheless, implicit learning still slightly outperforms the baseline for smaller samples. Moreover, the stability (risk difference) is almost not affected, hence the generalization capabilities of the implicit learning is still considerably superior.


\paragraph{Semantic Segmentation.} In the next experiment we study the behavior of implicit modeling for the application of semantic image segmentation. The task is to label each pixel of an image by a class from a predefined label-set $L$. Formally, the hidden vector $y$ is a labeling $y:V\rightarrow L$ of a graph $G=(V,E)$ whose nodes $i\in V$ correspond to image pixels and whose edges $(i,j)\in E$ encode a neighborhood structure (we use 8-neighborhood). The observation $x$ is an image $x:V\rightarrow C$, i.e.~mapping that assigns a color $c\in C$ to every pixel of the image. Among other things, in this experiment we would like to simulate a situation that is very common in Computer Vision. Assume that there exists already a model and we want to improve it using it just as a ``black box'', i.e.~to combine it with some other model or to build something on top of it. We use pixel-wise independent decision forest classifier as such a ``black box''\footnote{We used the code from \cite{dtf}.}. Let $z_i(x)\in L$ be the output label predicted by the decision forest for the pixel $i$. Based on that we define the following conditional random field to specify $p(Y|X)$:
\begin{equation}\label{eq:dtfcrf}
p(y|x)\propto \exp\Bigl[\sum_{i\in V} q\bigl(y_i,z_i(x)\bigr) + \sum_{ij\in E} \bigl(a_t(y_i,y_j)+b_t(y_i,y_j)\cdot ||x_i-x_j||^2\bigr)\Bigr]
\end{equation}
The first part of the energy represents unary potentials. The table $q:L{\times}L\rightarrow\mathbb R$ reflects the ``reliability'' of the decision forest's output. The pairwise potentials are linear functions of the squared color differences $a+b\cdot||x_i-x_j||^2$ in the corresponding neighbouring pixels. The coefficients $a$ and $b$ depend on edge type (horizontal, vertical or diagonal, referred by index $t$ in \eqref{eq:dtfcrf}) as well as on the label pair $(y_i,y_j)$ on this edge. To summarize, given a decision forest we have $9\times|L|^2$ additional free parameters in our CRF.

For the generative part (i.e.~the probability distribution $p(X|Y)$) we use a model that is similar in spirit to a Gaussian mixture model. To this end we introduce latent variables $g_i$ for each pixel $i$ that represent ``color numbers''. In addition the coloring is required to be smooth inside the segments. To summarize, the model is
\begin{eqnarray}\label{eq:dtfmrf}
\lefteqn{p(x,g|y)\propto\exp\Bigl[\sum_{i\in V} \bigl( h(y_i,g_i)+c\cdot||x_i||^2+}\nonumber\\
& & +\langle d(g_i),x_i\rangle\bigr)+e{\cdot}\sum_{ij\in E} \delta(y_i{=}y_j){\cdot}||x_i-x_j||^2 \Bigr] .
\end{eqnarray}
The weight matrix $h$ assigns a value for each pair (label, color number). The coefficient $c$ (shared for all colors) and color specific vectors $d$ control the distribution of RGB-values given the color number. The constant $e$ controls the smoothness of coloring inside the segments. $\delta(\cdot)$ is the Kronecker delta -- it gives $1$ if its argument is true. The necessary $p(x|y)$ is obtained by marginalizing \eqref{eq:dtfmrf} over $g$. All unknown parameters are learned in our experiments.

Recall that for learning we need to draw samples from $p(Y|X)$ and $p(X|Y)$, that is not a trivial task by itself, since both \eqref{eq:dtfcrf} and \eqref{eq:dtfmrf} are CRF-s. We use Gibbs Sampling for this purpose. In addition, in order to accelerate the overall learning procedure, we use ``warm start'' at each gradient update iteration as follows. Let e.g.~$\tilde y_t$ denotes the example $\tilde y$ sampled from the Markov chain for the training example $(x_t,y_t)=(x^\ast,y^\ast)$ (see Alg.~1). At the beginning of the learning procedure we initialize the necessary labeling examples from all chains (i.e.~$\tilde y_t$, $\hat y_t$ for all training examples $(x_t,y_t)$) randomly. The coloring examples $\tilde x_t$ are initialized by the corresponding original images $x_t$. At each gradient update iteration of the algorithm we draw new examples by Gibbs Sampling, starting from examples obtained in the previous gradient update iteration. For example, the new $\tilde y_t$ is obtained by Gibbs Sampling using the current conditional probability distribution $p(Y|X=x_t;\theta_1)$, starting from the old $\tilde y_t$ (from the previous iteration), performing thereby only a small number of Gibbs Sampling iterations. In practice, just one Gibbs Sampling iteration per gradient update usually performs best, taking into account both the quality of the results and computational demand. Similar acceleration tricks are widely used in the literature (see e.g.~\cite{tieleman2008training,ChenSYU15}).

We use the Stanford Background Dataset \cite{DAGS} for this experiment. Due to the computational demand of Gibbs Sampling, we do not use the whole dataset and only choose $100$ images randomly. The training examples (up to $90$ images) are chosen randomly out of these $100$, the rest is used for testing. For each training sample we first learn a pixel-independent Decision Forest classifier using standard methods. For the corresponding CRF \eqref{eq:dtfcrf} we again compare two different scenarios: standard conditional likelihood learning and implicit learning. After the models are learnt we use the maximum marginal decision strategy for inference. The ``recognition error'' is the relative number of misclassified pixels, i.e.~the Hamming distance averaged over the training/test sample. The measures of interest are as before the recognition error itself and the averaged absolute difference between the training and the test errors.

\begin{figure}
\includegraphics[width=0.49\linewidth]{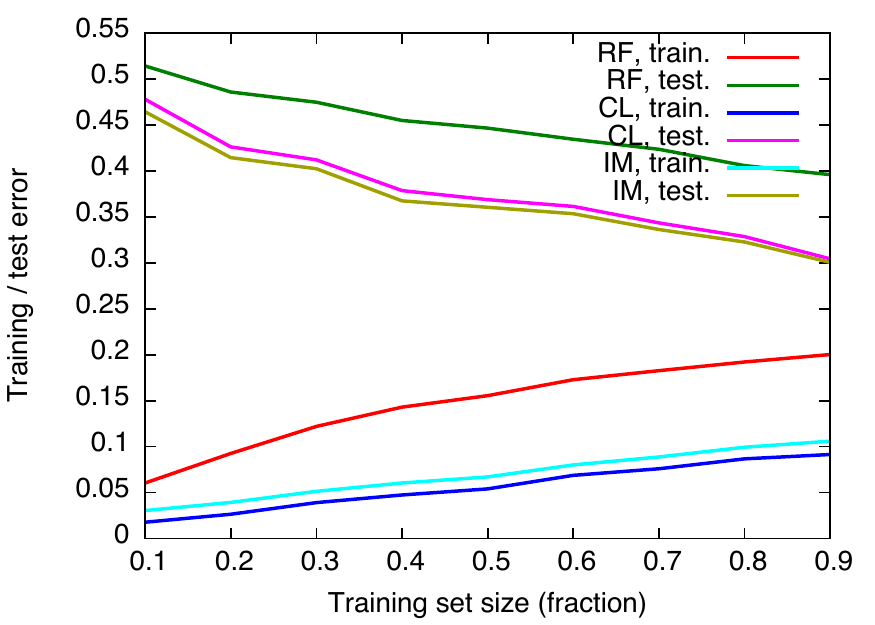}\hfill
\includegraphics[width=0.49\linewidth]{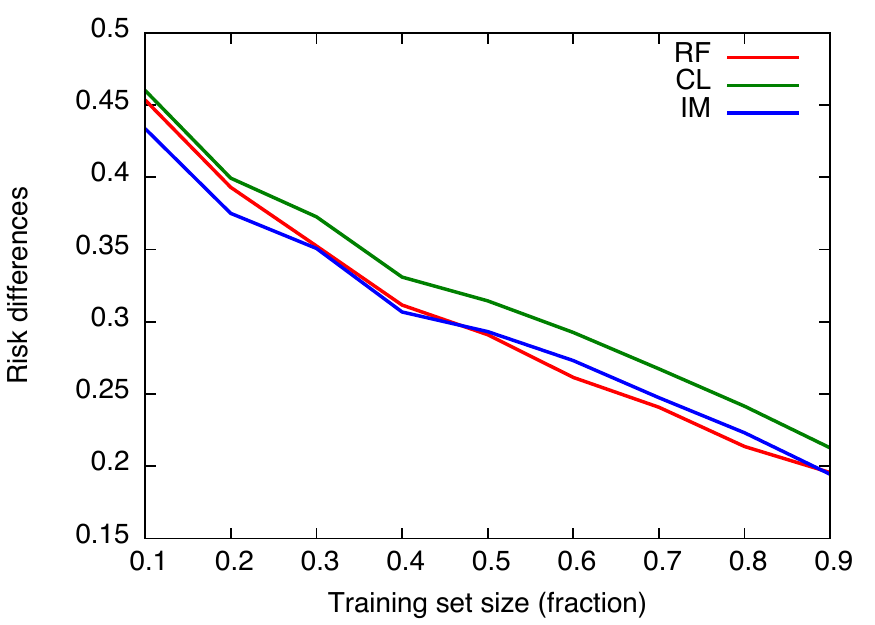}
\caption{\label{fig:su} Results for semantic segmentation. RF -- pixel-wise independent Random Forest classifier. CL -- Conditional Likelihood learning. IM -- implicit model. Left: training and test errors. Right: error differences.}
\end{figure}

The results are presented in Fig.~\ref{fig:su} (some qualitative results are shown in appendix \ref{appendix3}). Despite the test error we give also the training error for clarity. In addition, the rates for basic Random Forest classifier are also given. The typical behaviour can be observed. Building a CRF on top of the Random Forest improves both training and test errors. However, the training error is improved more essentially as compared to the test error. Hence, the stability becomes worse (compare the red and the green lines in Fig.~\ref{fig:su}, right). It is indeed expected, because in the CRF additional free parameters are introduced. Using implicit modeling makes the {\em training} error {\em worse} as compared to the Conditional Likelihood. However, at the same time the {\em test} error becomes {\em better}. Hence, by using implicit modeling we get the stability back. To summarize, the implicit learning outperforms the standard Conditional Likelihood with respect to both recognition rate and stability for the whole range of training sample sizes.


\section{Conclusion}

We presented a new modeling approach called implicit modeling. The families of joint probability distributions of interest are modeled implicitly by specifying two families of conditional probability distributions. We showed that the weak implicit modeling is a generalization of both generative and discriminative approaches. Being able to use powerful discriminative models as its constituents, the method has at the same time better generalization capabilities. Experiments on both artificial and real data confirm the statement.

There are numerous directions for further works. The presented approach can work with nearly arbitrary constituents. Hence, an interesting research direction would be to apply the method to complex discriminative models such as decision tree fields \cite{dtf} or convolutional neural networks \cite{long_shelhamer_fcn}. 

In this paper we mainly focused on modeling aspects and used maximum likelihood for learning. Other choices are possible as well. For instance, using pseudo-likelihood may lead to more efficient learning. It would be also interesting to apply other learning objectives like loss-driven learning or composite likelihood to implicitly defined models.

In our opinion, an essential advantage of the implicit modeling in comparison to standard discriminative models is the ability (at least theoretically) to learn unknown parameters in a fully unsupervised manner (since the {\em joint} likelihood is optimized). It gives the possibility to train complex posterior probability distributions using incomplete or very weakly labeled data, that is often highly desirable in practice. We did not elaborate this direction so far, hence, it will be one of our further works.

\bibliography{impl-mod-arxiv}
\bibliographystyle{plain}

\newpage 

\begin{appendix}
\setcounter{mydef}{0}
\setcounter{mylem}{0}

\section{Proof of Lemma 1}\label{appendix1}

Let us first recall the definitions and the statement.
\begin{mydef}[Implicit Modeling]
Let two families $p(Y|X;\Theta_1)$ and $p(X|Y;\Theta_2)$ of conditional probability distributions be given. A joint probability distribution $p(X,Y)$ belongs to the family $p(X,Y;\Theta_1,\Theta_2)$ of interest if there exist $p(X)$, $p(Y)$, $\theta_1\in\Theta_1$ and $\theta_2\in\Theta_2$ satisfying
\begin{equation}\label{eq:impl1_a}
p(x)\cdot p(y|x;\theta_1)=p(y)\cdot p(x|y;\theta_2) \ \ \forall x\in X, y\in Y.
\end{equation}
Given that, the joint probability distribution is defined by $p(x,y;\theta_1,\theta_2)=p(x)\cdot p(y|x;\theta_1)$.
\end{mydef}

{\sloppy
\begin{mydef}[Weak Implicit Modeling]
Let two families $p(Y|X;\Theta_1)$ and $p(X|Y;\Theta_2)$ of conditional probability distributions be given. A joint probability distribution $p(X,Y)$ belongs to the family $p(X,Y;\Theta_1,\Theta_2)$ of interest if there exist $p(X)$, $p(Y)$, $\theta_1\in\Theta_1$ and $\theta_2\in\Theta_2$ satisfying
\begin{eqnarray}\label{eq:impl2_a}
& & p(y)=\sum_{x\in X}p(x)\cdot p(y|x;\theta_1), \ \ \forall\ y\in Y , \nonumber \\
& & p(x)=\sum_{y\in Y}p(y)\cdot p(x|y;\theta_2), \ \ \forall\ x\in X .
\end{eqnarray}
Given that, the joint probability distribution is defined either by $p(x,y;\theta_1,\theta_2){=}p(x){\cdot} p(y|x;\theta_1)$ or by $p(x,y;\theta_1,\theta_2){=}p(y){\cdot} p(x|y;\theta_2)$ depending on the application.
\end{mydef}
}

\begin{mylem}
Under some mild conditions there exists a pair of marginal probability distributions $p(X)$ and $p(Y)$ satisfying \eqref{eq:impl2_a} for any pair of conditional probability distributions $p(Y|X;\theta_1)$ and $p(X|Y;\theta_2)$.
\end{mylem}
\begin{proof}
Let us represent the probability distribution $p(X)$ by a vector $v_x{\in}\mathbb R^{|X|}$ and the probability distribution $p(Y)$ by a vector $v_y{\in}\mathbb R^{|Y|}$. Similarly, $p(Y|X)$ can be understood as a matrix $A$ of size $|Y|{\times}|X|$, where matrix elements $a_{ij}$ satisfy $a_{ij}\geq 0$ and $\sum_i a_{ij}=1$ for all $j$. Analogously, $p(X|Y)$ is another matrix $B$ of $|X|{\times}|Y|$ elements. Hence, conditions \eqref{eq:impl2_a} can be rewritten as
\begin{equation}\label{eq:cond}
v_y=A\cdot v_x, \ \ \ v_x = B\cdot v_y .
\end{equation}
Let us substitute the first part of \eqref{eq:cond} in the second one and obtain 
\begin{equation}\label{eq:cond1}
v_x=B\cdot A\cdot v_x=C\cdot v_x.
\end{equation}
Note that both $A$ and $B$ are stochastic matrices. Hence, so is their product $C{=}B{\cdot}A$ as well. It means that under some mild conditions (for example, it is enough that all elements of $C$ are strictly positive that is almost always the case in practice) there exists $v_x$ satisfying \eqref{eq:cond1}, which is the eigenvector of $C$ to the eigenvalue $1$. To summarize, let us define $v_x$ to be the eigenvector of $C$, and  $v_y$ as $v_y=A\cdot v_x$. Obviously, all conditions \eqref{eq:impl2_a}-\eqref{eq:cond1} are satisfied.

\end{proof}

\noindent
{\bf Corollary}. Let us consider the Matrix $C{=}B{\cdot}A$ as a matrix of transition probabilities of a Markov chain, that generates sequences of $x$-s $(x^0,x^1,\ldots,x^n)$, $n{\rightarrow}\infty$. Obviously, the above considered vector $v_x$ (and hence, the probability distribution $p(X)$) is the stationary distribution of this Markov chain. On the other hand, let us consider now the matrix $D{=}A{\cdot}B$ that can be understood as the matrix of transition probabilities for a Markov chain, that generates sequences of $y$-s $(y^0,y^1,\ldots,y^n)$, $n{\rightarrow}\infty$. Obviously, the eigenvector of $D$ also satisfies all conditions \eqref{eq:impl2_a}-\eqref{eq:cond1}, and hence, is also the stationary distribution of the corresponding Markov chain. Let us now consider a Markov chain that generates sequences $(x^0,y^0,x^1,y^1,\ldots, x^{n-1},y^{n-1},x^n,y^n)$, $n{\rightarrow}\infty$ by alternate sampling from the corresponding distributions $p(Y|X)$ and $p(Y|X)$ respectively, i.e.~$y^i$ is drawn from $p(Y|X{=}x^i)$ and $x^i$ is drawn from $p(X|Y{=}y^{i-1})$. It is easy to see that both $x^n$ and $y^n$ obey the corresponding stationary distributions $p(X)$ and $p(Y)$ respectively.


\section{Algorithm for Learning Implicit Models}
\label{appendix2}

Here we give a more detailed derivation of the algorithm for learning implicit models. First, let us recall the problem to be solved. Given a training sample of observation-label pairs $\bigl((x_t,y_t),t{=}1{\ldots}T\bigr)$, we aim to learn our implicit model by maximizing the joint log-likelihood
\begin{eqnarray}
\lefteqn{\arg\max_{\theta_1,\theta_2} \sum_{t=1}^T \log p(x_t,y_t;\theta_1,\theta_2) =} \nonumber \\
& & =\arg\max_{\theta_1,\theta_2}\sum_{t=1}^T \Bigl[\log p(x_t;\theta_1,\theta_2)+ \log p(y_t|x_t;\theta_1)\Bigr] .
\label{eq:task2_a}
\end{eqnarray}
For gradient-based optimization, we need to differentiate \eqref{eq:task2_a} wrt.~$\theta_1,\theta_2$. It is straightforward to do that for the conditional $p(y|x;\theta_1)$. In contrast, the differentiation of the first addend in \eqref{eq:task2_a} is less trivial because there is no closed-form expression for it, i.e.~the marginal distribution $p(X;\theta_1,\theta_2)$ depends {\em implicitly} on both $\theta_1$ and $\theta_2$.

We model $p(X;\theta_1,\theta_2)$ by means of a generation process that samples observations according to the desired probability distribution. Let us assume an (infinite) Markov chain that generates sequences $(x^0,y^0,x^1,y^1,\ldots, x^{n-1},y^{n-1},x^n)$, $n{\rightarrow}\infty$ by sampling from the corresponding conditional distributions $p(Y|X;\theta_1)$ and $p(Y|X;\theta_2)$ respectively, i.e.~$y^i$ is drawn from $p(Y|X{=}x^i;\theta_1)$ and $x^i$ is drawn from $p(X|Y{=}y^{i-1};\theta_2)$. Obviously, stationary distribution of such a Markov chain satisfies \eqref{eq:impl2_a} (see the Corollary above), i.e.~our $p(X;\theta_1,\theta_2)$ is the stationary distribution of this Markov chain. Hence, we model $p(x_t;\theta_1,\theta_2)$ with the probability that a Markov chain as described above generates the training example $x_t$.

To ease exposition, let us assume for now that we want to maximize the probability of a single observation $x^\ast$ under $p(X;\theta_1,\theta_2)$, i.e.~only the first part of \eqref{eq:task2_a} for just one example. Consider a sequence 
\begin{equation}
(z,x^\ast) = (x^0, y^0, x^1, y^1, \ldots , x^{n-1}, y^{n-1}, x^\ast)
\end{equation}
generated by the Markov chain whose last entry is the observation $x^\ast$ ($z$ summarizes all generated $x$-s and $y$-s but $x^\ast$, $Z$ denotes the set of all sequences). The probability of this sequence is
\begin{eqnarray}
\label{eq:chain_a}
\lefteqn{p(z,x^\ast;\theta_1,\theta_2) = p(x^0) \cdot p(y^0|x^0;\theta_1) \cdot}\nonumber \\
& & \cdot \left[\prod_{i=1}^{n-1} p(x^i|y^{i-1};\theta_2) \cdot p(y^i|x^i;\theta_1)\right] \cdot p(x^\ast|y^{n-1};\theta_2) .
\end{eqnarray}
The probability of the marginal for the observation $x^\ast$ is obtained by marginalization over all possible sequences that end with $x^\ast$:
\begin{equation}
p(x^\ast;\theta_1,\theta_2) = \sum_z p(z,x^\ast;\theta_1,\theta_2) .
\end{equation}
Note that this probability does not depend on $x_0$ in \eqref{eq:chain_a} if the chain has infinite length, and can be neglected in practice if it is run long enough. Note that now $p(x^\ast)$ depends on the unknown parameters $\theta_1$ and $\theta_2$ in an {\em explicit} manner through the generating Markov chain. Hence, we can compute the gradient of the log marginal
\begin{eqnarray}
\label{eq:llh_marginal_a}
\lefteqn{\frac{\partial}{\partial \theta_{1,2}} \log p(x^\ast;\theta_1,\theta_2) = \frac{\partial}{\partial \theta_{1,2}} \log \sum_z p(z,x^\ast;\theta_1,\theta_2) = } \nonumber \\
& & = \left[\sum_{z'} p(z',x^\ast;\theta_1,\theta_2)\right]^{-1}\cdot
\sum_z \frac{\partial p(z,x^\ast;\theta_1,\theta_2)}{\partial \theta_{1,2}} = \nonumber\\
& & = \sum_z \frac{p(z,x^\ast;\theta_1,\theta_2)}{p(x^\ast;\theta_1,\theta_2)}\cdot
\frac{\partial \log p(z,x^\ast;\theta_1,\theta_2)}{\partial \theta_{1,2}} = \nonumber\\
& & = \sum_z p(z|x^\ast;\theta_1,\theta_2)\cdot\frac{\partial \log p(z,x^\ast;\theta_1,\theta_2)}{\partial \theta_{1,2}}
\end{eqnarray}
that we require for learning \eqref{eq:task2_a}. In practice, marginalization over all sequences $z$ in \eqref{eq:llh_marginal_a} is intractable, especially when $x$ and $y$ are complex (e.g.~images or labelings). Note however that the gradient in \eqref{eq:llh_marginal_a} is an expectation over the probability distribution $p(Z|X^n{=}x^\ast;\theta_1,\theta_2)$. Hence, we can use {\em stochastic gradient ascent}, i.e.~we replace the expectation of a random variable by its realization. To summarize, we need to carry out two steps to perform stochastic gradient ascent for \eqref{eq:llh_marginal_a}:
{\sloppy 
\begin{enumerate}
\item Generate a sequence $\hat{z}$ according to the probability distribution $p(Z|X^n{=}x^\ast;\theta_1,\theta_2)$ given the current model parameters, and
\item compute the gradient of $\log p(\hat{z},x^\ast;\theta_1,\theta_2)$ with respect to $\theta_1$, $\theta_2$.
\end{enumerate}
}
Concerning the second step, we obtain
\begin{eqnarray}
\frac{\partial \log p(\hat{z},x^\ast;\theta_1,\theta_2)}{\partial \theta_1} & = & \sum_{i=0}^{n-1}\frac{\partial\log p(\hat{y}^i|\hat{x}^i;\theta_1)}{\partial\theta_1}, \nonumber \\
\frac{\partial \log p(\hat{z},x^\ast;\theta_1,\theta_2)}{\partial \theta_2} & = & \sum_{i=1}^{n-1}\frac{\partial\log p(\hat{x}^i|\hat{y}^{i-1};\theta_2)}{\partial\theta_2} + \nonumber \\
& & + \frac{\partial \log p(x^\ast|\hat{y}^{n-1};\theta_2)}{\partial\theta_2} ,
\end{eqnarray}
where $\hat{x}^i$ and $\hat{y}^i$ are elements of the generated $\hat{z}$. This can be interpreted as follows. After a sequence $\hat{z}$ is generated, its pairs $(\hat{x}^i,\hat{y}^i)$ at ``odd'' positions can be considered as an ``additional labeled training sample'', for which the conditional log-likelihood of $p(Y|X;\theta_1)$ has to be maximized with respect to $\theta_1$. The same holds for ``even'' positions of the chain (including the transition from $\hat{y}^{n-1}$ to $x^\ast$) and the conditional probability distribution $p(X|Y;\theta_2)$ wrt.~$\theta_2$.

Unfortunately, generating a sequence $\hat{z}$ for the first step is not as easy as just to generate a sequence from a Markov chain. This is because we need to generate sequences according to $p(Z|X^n{=}x^\ast;\theta_1,\theta_2)$, i.e.~conditioned by the last chain member. Common approaches for that would be e.g.~importance sampling, sampling with rejection or Gibbs sampling. Unfortunately, these methods usually turn out to be too time-consuming, especially for complex $x$ and $y$. Therefore in practice, we assume that our Markov chain is ``reversible'' in the sense that we can generate the sequence in the ``opposite'' direction, i.e.~starting from $x^\ast$ towards $x^0$, with the same probability. Note that this assumption is true, if the current conditional probability distributions $p(Y|X;\theta_1)$ and $p(X|Y;\theta_2)$ lead to existence of a {\em unique} joint probability distribution $p(X,Y;\theta_1,\theta_2)$, i.e.~if the conditions \eqref{eq:impl1_a} hold and we are in the case of (non-weak) implicit modeling. Otherwise, we assume that the two joint probability distributions induced by \eqref{eq:impl2_a} (see the definition of the weak implicit modeling) are close enough to each other. Hence, examples drawn from $p(x,y;\theta_1,\theta_2){=}p(x){\cdot} p(y|x;\theta_1)$ and from $p(x,y;\theta_1,\theta_2){=}p(y){\cdot} p(x|y;\theta_2)$ obey nearly the same joint probability distribution. Observe that in doing so the pairs $(\hat{x}^i,\hat{y}^i)$ from the sequence $\hat{z}$ that were generated using $p(X|Y;\theta_2)$ (during the generation in opposite direction) serve as the training data to learn $p(Y|X;\theta_1)$ and vice versa.

\paragraph{Exponential family.} In order to do learning as described above, the conditional probabilities $p(y|x;\theta_1)$ and $p(x|y;\theta_2)$ must be differentiable wrt.~$\theta_1$ and $\theta_2$ respectively. Furthermore, we should be able to draw samples from both distributions. To make this slightly more concrete, let us assume that the conditional distributions both belong to the exponential family, i.e.~they can be written in the form
\begin{align}
p(y|x;\theta_1) &= \frac{1}{Z_1(x,\theta_1)} \cdot \exp\bigl(\langle\eta_1(x,y),\theta_1\rangle\bigr) , \nonumber \\
p(x|y;\theta_2) &= \frac{1}{Z_2(y,\theta_2)} \cdot \exp\bigl(\langle\eta_2(x,y),\theta_2\rangle\bigr),
\end{align}
where $Z_{1,2}$ are the partition functions, $\eta_{1,2}$ are the sufficient statistics, and $\theta_{1,2}$ the unknown parameters; $\langle\cdot,\cdot\rangle$ denotes the inner product between two vectors. For a particular pair $(x,y)$ the gradients of the conditional log-likelihoods are
\begin{eqnarray}\label{eq:lcl_t2}
\lefteqn{\frac{\partial\log p(y|x;\theta_1)}{\partial \theta_1} = \eta_1(x,y)-\frac{\partial\log Z_1(x,\theta_1)}{\partial \theta_1} =} \nonumber \\
& & = \eta_1(x,y)-\mathbb E_{p(Y|X;\theta_1)}\bigl[\eta_1(x,y)\bigr],\nonumber\\
\nonumber \\
\lefteqn{\frac{\partial\log p(x|y;\theta_2)}{\partial \theta_2} = \eta_2(x,y)-\frac{\partial\log Z_2(y,\theta_2)}{\partial \theta_2} =} \nonumber \\
& & = \eta_2(x,y)-\mathbb E_{p(X|Y;\theta_2)}\bigl[\eta_2(x,y)\bigr],
\end{eqnarray}
where $\mathbb E_p[\cdot]$ denotes the expectation of a random variable $[\cdot]$ over a probability distribution $p$. Similar to the marginalization over all sequences $z$ (see above), we use stochastic approximation, i.e.~we exchange the expectation of a random variable by its realization. Hence, the stochastic gradient for $p(y|x;\theta_1)$ in \eqref{eq:lcl_t2} is obtained as the difference of sufficient statistics $\eta_1(x,y)-\eta_1(x,\hat{y})$ with an example $\hat{y}$ drawn from the probability distribution $p(Y|X;\theta_1)$ (for $p(X|Y;\theta_2)$ analogously). Note that we already need to draw these examples in order to generate the chain $(\hat{z},x^\ast)$, so we can use them for gradient calculation \eqref{eq:lcl_t2}, i.e.~we do not need to extra generate them.

So far we focused on maximization of $p(x^\ast)$, i.e.~only the first addend in \eqref{eq:task2_a}. Obviously, we should take into account the second addend in \eqref{eq:task2_a} as well. Its stochastic gradient however is again the difference of the corresponding sufficient statistics \eqref{eq:lcl_t2}. So we have only to add this gradient part to the previous one. 

To conclude, one stochastic gradient ascent step for one training example $(x_t,y_t)$ is summarized in Algorithm \ref{alg:alg1} and illustrated in Fig.~\ref{alg:alg} (see the main part of the paper). It is easy to see that the influence of the generated pairs on the log-likelihood decreases with the distance between the example position and the position of $x^\ast$ in the chain. In practice we find that it is enough to generate a chain of minimal length, as required to compute the necessary sufficient statistics. 

\section{Qualitative Results for Semantic Segmentation}
\label{appendix3}

Here we would like just to give an impression about qualitative results of our approach applied for semantic segmentation. In Figs.~\ref{fig:semseg1} and \ref{fig:semseg2} chains of examples generated during the learning of implicit models are presented (see also Fig.~\ref{alg:alg} in the main part of the paper). The last column is the max-marginal decision for the corresponding image.

\newcommand{\imrowaa}[1]{
\noindent
\includegraphics[width=0.16\textwidth]{./pics/#1_6_y2_c.png}\hfill
\includegraphics[width=0.16\textwidth]{./pics/#1_4_x1.png}\hfill
\includegraphics[width=0.16\textwidth]{./pics/#1_3_y1_c.png}\hfill
\includegraphics[width=0.16\textwidth]{./pics/#1_1_x0.png}\hfill
\includegraphics[width=0.16\textwidth]{./pics/#1_0_y0_c.png}\hfill
\includegraphics[width=0.16\textwidth]{./pics/#1_crf.png}
}

\begin{figure*}
\begin{center}
\imrowaa{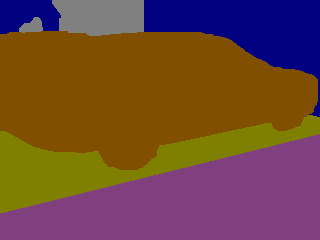}
\imrowaa{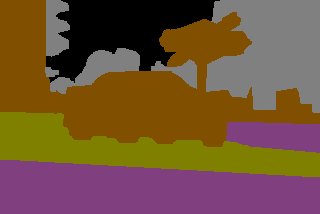}
\imrowaa{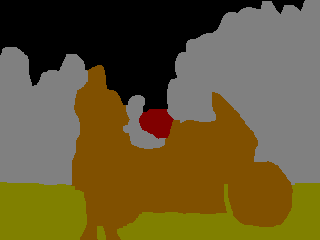}
\imrowaa{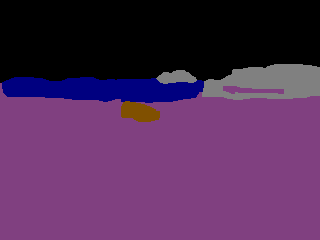}
\imrowaa{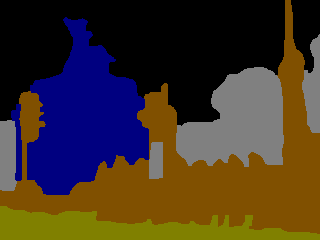}
\imrowaa{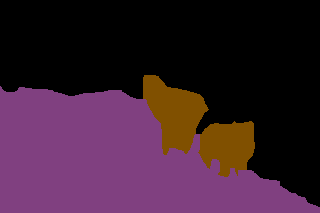}
\imrowaa{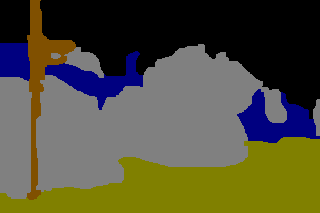}
\imrowaa{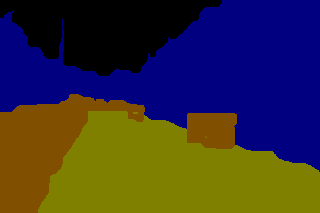}
\imrowaa{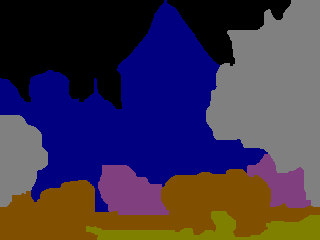}
\imrowaa{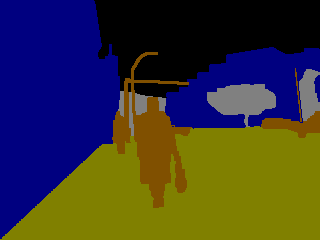}
\end{center}
\caption{\label{fig:semseg1} Chains of examples generated during the learning of implicit models. The first three columns are $\hat y$, $\tilde x$ and $\tilde y$ respectively (see Alg.~\ref{alg:alg1}). The fourth and fifth columns are samples from the training set, i.e.~the observation $x^\ast$ and the ground truth $y^\ast$. The last column is the max-marginal decision for the corresponding image.}
\end{figure*}

\begin{figure*}
\begin{center}
\imrowaa{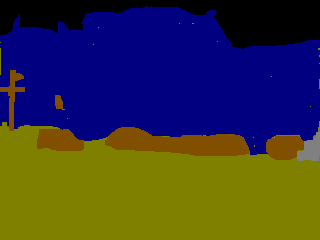}
\imrowaa{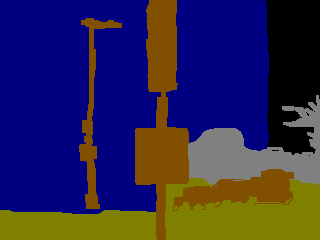}
\imrowaa{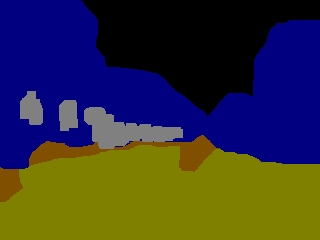}
\imrowaa{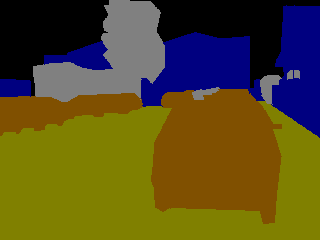}
\imrowaa{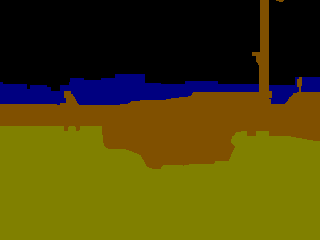}
\imrowaa{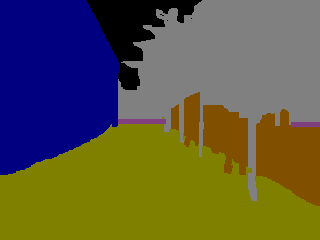}
\imrowaa{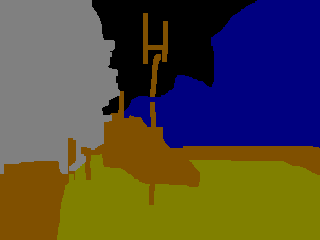}
\imrowaa{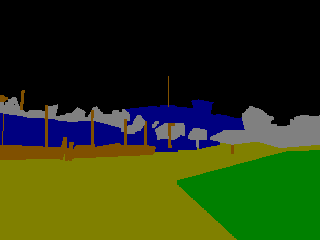}
\imrowaa{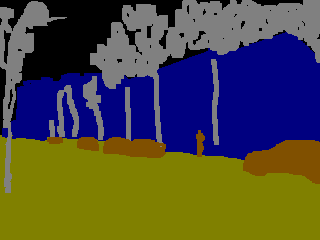}
\imrowaa{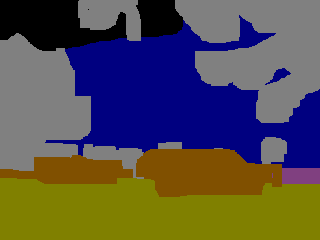}
\end{center}
\caption{\label{fig:semseg2} Chains of examples generated during the learning of implicit models. The first three columns are $\hat y$, $\tilde x$ and $\tilde y$ respectively (see Alg.~\ref{alg:alg1}). The fourth and fifth columns are samples from the training set, i.e.~the observation $x^\ast$ and the ground truth $y^\ast$. The last column is the max-marginal decision for the corresponding image.}
\end{figure*}
\end{appendix}
\end{document}